\newcommand{\CLASSINPUTtoptextmargin}{19.1mm}
\newcommand{\CLASSINPUTbottomtextmargin}{25.7mm}
\documentclass[conference]{IEEEtran}
\usepackage{times}
\usepackage[numbers]{natbib}
\usepackage{multicol}
\usepackage[bookmarks=true,breaklinks]{hyperref}
\usepackage{rigidmath}
\usepackage{amsmath}
\usepackage[dvipsnames]{xcolor}
\usepackage{graphicx}
\PassOptionsToPackage{hyphens}{url}
\usepackage{breakurl}
\usepackage[hyphens]{url}
\usepackage{comment}
\usepackage{ragged2e}
\usepackage{subcaption}

\begin{document}

\title{Modeling and Analysis of Non-unique Behaviors in Multiple Frictional Impacts}

\author{Mathew Halm and Michael Posa \\ GRASP Laboratory,  University of Pennsylvania \\ \texttt{\{mhalm, posa\}@seas.upenn.edu} }

\maketitle

\begin{abstract}
Many fundamental challenges in robotics, based in manipulation or locomotion, require making and breaking contact with the environment.
To represent the complexity of frictional contact events, impulsive impact models are especially popular, as they often lead to mathematically and computationally tractable approaches.
However, when two or more impacts occur simultaneously, the precise sequencing of impact forces is generally unknown, leading to the potential for multiple possible outcomes.
This simultaneity is far from pathological, and occurs in many common robotics applications.
In this work, we propose an approach for resolving simultaneous frictional impacts, represented as a differential inclusion. 
Solutions to our model, an extension to multiple contacts of Routh's method, naturally capture the set of potential post-impact velocities.
We prove that solutions to the presented model must terminate.
This is, to the best of our knowledge, the first such guarantee for set-valued outcomes to simultaneous frictional impacts.
\end{abstract}

\IEEEpeerreviewmaketitle

\section{Introduction}
Modern robots are fast and strong, and, in some situations, their capabilities eclipse those of humans.
However, when these robots interact with their environment, whether by manipulating objects or traversing over uneven surfaces, they do so with far less skill than their human counterparts.
%The speed and accuracy that robots display in static or structured settings disappears when faced with less structured tasks.
Critical challenges facing the field lie in modeling, planning, and control of robots in these complex, multi-contact settings, particularly for locomotion \cite{Wieber2016} and manipulation \cite{Kemp2007}.

Rigid-body models of dynamics and contact (see \citet{Stewart2000} or \citet{Brogliato99} for an overview) are widely used in robotics, as they can lead to far more tractable methods than approaches which explicitly attempt to capture the stiff interaction between objects.
These approaches have also led to complementarity-based simulation schemes, such as \cite{Anitescu97, Drumwright2010, Horak2019, Kaufman08, Smith2012, Stewart1996a} and others.
Recent research, using complementarity models, has also been conducted into multi-contact optimal planning \cite{Mordatch15, Posa13, Posa2016} and control \cite{Hogan2016,Posa14}.
Similar applications have been seen for manipulation (e.g. \cite{Salehian2018}), including quasi-static approaches \cite{Chavan-Dafle2017, Halm2018}.
When impacts occur, rigid-body models approximate the event as an instantaneous change in velocity due to an impulsive force.

The approaches above, now deeply ingrained within the robotics community, universally assume that it is possible to determine a single potential post-impact velocity, even during simultaneous multi-contact. 
However, as observed in \cite{Hurmuzlu1994, Ivanov1995, Smith2012, Uchida2015, Wang1992} and others, including recent analysis of robot locomotion \cite{Remy2017}, the resolution of simultaneous impacts is dependent upon the sequence in which they are resolved.
Simulation schemes to this problem (e.g. \cite{Coumans2015, Erleben2004, Jia2013, Kaufman08, Liu2008, Smith2012, Uchida2015} and many others) focus on generation of a single solution via a heuristic (symmetry \cite{Kaufman08}, potential energy \cite{Uchida2015}, etc.).
However, for many practical applications in robotics, it is not possible to create a model detailed enough to reliably disambiguate between the multiple potential solutions;
essentially, the disambiguation performed by common simulation schemes is not grounded in physical principles.
Even were we to be given such detail, this lack of uniqueness often represents an extreme sensitivity to initial conditions: slight perturbations in the initial state of the system might lead to different impact sequences.
As a result, rather than focus on producing a \textit{single} potential solution, here we consider the \textit{set} of such solutions.

As the motivating examples in \ref{section:examples} will demonstrate, simultaneous impacts are not limited to unlikely, pathological events but are, in fact, regular occurrences in robotics and require careful analysis.
From the perspective of planning, learning, and control, it is critical to understand the role of this non-uniqueness (alternatively, extreme sensitivity), as some of the broad challenges in executing dynamic, multi-contact motion likely arise from these issues.
For example, methods which use a simulator to learn or plan a motion may, unwittingly, be planning for an ambiguous, therefore unstable, outcome due to multi-contact.
Furthermore, as the set of these ambiguous outcomes is often non-convex, it is insufficient to try to capture this sensitivity via simple models of uncertainty.
%\mhcomment{I'm not 100\% sure how to best express it but I think it'd be good to really push that this non-uniqueness / extreme sensitivity is not just a mathematical quirk of the models, but is also something observable in real systems.}
                     
Many methods have been proposed for modeling single impacts (e.g. \cite{Bhatt95, Chatterjee98, Routh91, Stronge90}, and others) along with recent data-driven models \cite{Fazeli2017, Jiang2018}, experimental validation \cite{Fazeli2017a}, and efforts to translate multi-contact simulated motions to real robots \cite{Tan2018}.
Comparable results for simultaneous impacts have largely focused on simulation, with the intent to produce a single, reasonable solution (e.g \cite{Anitescu97, Drumwright2010, Smith2012}), where \citet{Anitescu97} and \citet{Drumwright2010} guarantee termination of their numerical methods.
Other related work addresses specialized, restricted settings. 
\citet{Seghete2014} developed a model where solutions were guaranteed to exist, but assumed that contact normal vectors are linearly independent.
\citet{Burden2016} studied discontinuous vector fields, with strong results and applications to robot impacts, but are similarly restricted to frictionless contact.
\citet{Johnson2016a} treated a limited form of friction, but assumed that contact occurs only at massless limbs.
For a quasi-static model, thus without impact, \citet{Halm2018} guaranteed existence of solutions for multi-contact motion.
                     
This work extends Routh's  graphical model \cite{Routh91} to address simultaneous, inelastic impacts by permitting impulses to occur in arbitrary sequences.
As a result, the model produces a set-valued map that captures the inherent lack of uniqueness.
We believe this is the appropriate description for robotic planning and control, as motions that present as non-unique will, for physical systems, display extreme sensitivity to any errors in estimation or control.
In contrast with prior literature, the presented model captures a broad class of frictional systems.
In \ref{section:model}, we describe the model and a number of its theoretical properties and in \ref{section:termination} we prove the key result that the impact model is guaranteed to terminate.
To the best of the authors' knowledge, this work presents the first known formal result for set-valued solutions to simultaneous frictional impact.
\section{Background}
We now introduce notation for and study the limiting behaviors of the frictional impact dynamics of rigid multibody systems. Denote the interior, closure, and convex hull of a set $A$ as $\Interior(A)$, $\Closure(A)$, and $\Hull(A)$. We identify the $l_p$-norm and unit direction of a vector $\Velocity \in \Real^n$ as $\Norm{\Velocity}_p$ and $\Direction{\Velocity} = \frac{\Velocity}{\TwoNorm{\Velocity}}$, respectively. We define the open $r$-radius ball in $\Real^n$ as $\Ball[r]$.
We denote $\Real^{n+} \subseteq \Real^n$ as the vectors with strictly positive entries and define a function $f: \Domain \subseteq \Real^n \to \Closure \Real^+$ to be positive definite if it is strictly positive on $\Domain \setminus \Braces{\ZeroVector}$. For a single-valued function $f: A \to B$ and a set-valued function $D: A \to \PowerSet{B}$, we denote the image of $A' \subseteq A$ under $f$ and $D$ as $f(A') \subseteq B$ and $D(A') \subseteq B$ respectively.
\subsection{Functional Analysis}
The results herein are broadly derived from measure theory and functional analysis; for a thorough background, see \citet{Rudin1986,Rudin1991}.
For a set $\Domain \subseteq \Real^n$, we equip $\Domain$ with the standard Euclidean metric and norm, and integrals on $\Domain$ are with respect to the Lebesgue measure by default.
The total time derivative $\dot{\vect f}(t)$ of an absolutely continuous function $\vect f(t)$ is taken in the Lebesgue sense (i.e. ${\vect f}(t)$  is the anti-derivative of $\dot{\vect f}(t)$, which is defined almost everywhere ($a.e.$)). Convergence of a sequence of functions $f_n$ to $f$ almost everywhere and uniformly are denoted $f_n \AlmostConvergence f$ and $f_n \UniformConvergence f$, respectively. A key result for the derivations in this work is the Arzel\`a-Ascoli Theorem \cite{Rudin1991}:

\begin{theorem}[Arzel\`a-Ascoli]\label{thm:Rellich}
	Let $\Sequence{\vect f}{n}$ be a uniformly bounded equicontinuous sequence of $\Real^n$-valued functions on some compact interval $I$. Then there exists a function $\vect f$ and subsequence $\Parentheses{{\vect f}_{n_k}}_{k \in \Natural}$ such that $\vect f_{n_k} \UniformConvergence \vect f$.
\end{theorem}
\begin{figure*}
\centering
\begin{minipage}{.49\hsize}
	\vspace{2mm}
    \begin{subfigure}{.45 \hsize}
	   \includegraphics[width=\hsize]{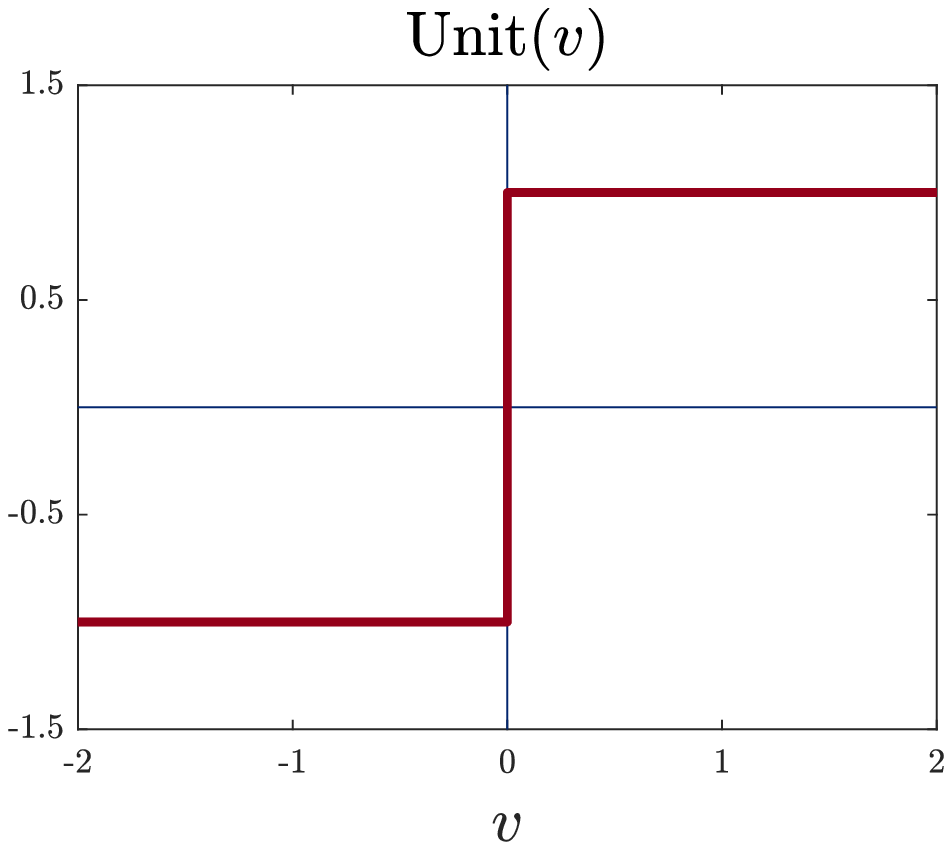}
       \caption{\label{figure_Ua}}
    \end{subfigure}
    \hfill
    \begin{subfigure}{.47 \hsize}
        \includegraphics[width=\hsize]{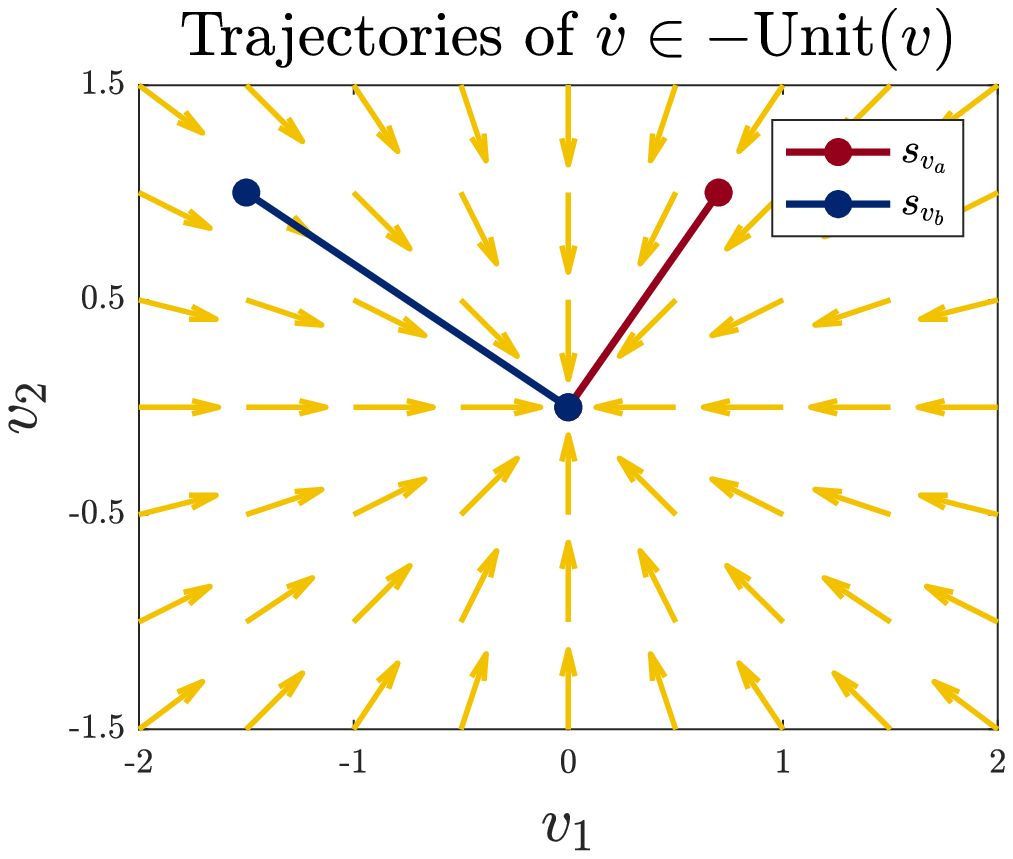}
        \caption{\label{figure_Ub}}
    \end{subfigure}
	\caption{(\subref*{figure_Ua}) Graph of $\Unit(\Velocity)$ for $n=1$. Note that $\Unit(\Velocity)$ is continuous on $\Velocity \neq 0$. At $\ZeroVector$, $\Unit$ takes the value $\Brackets{-1,1}$, which contains a continuous extension of $\Direction{\Velocity}$ from both the left $(-1)$ and the right $(+1)$, so that $\Unit$ is u.s.c.. (\subref*{figure_Ub}) Flow field of the solutions to $\dot \Velocity \in -\Unit(\Velocity)$ for $n=2$.}
	\label{fig:unit}
\end{minipage}
\hfill
\begin{minipage}{.48\hsize}
		\centering
		\vspace{1mm}
        \includegraphics[width=0.7\textwidth]{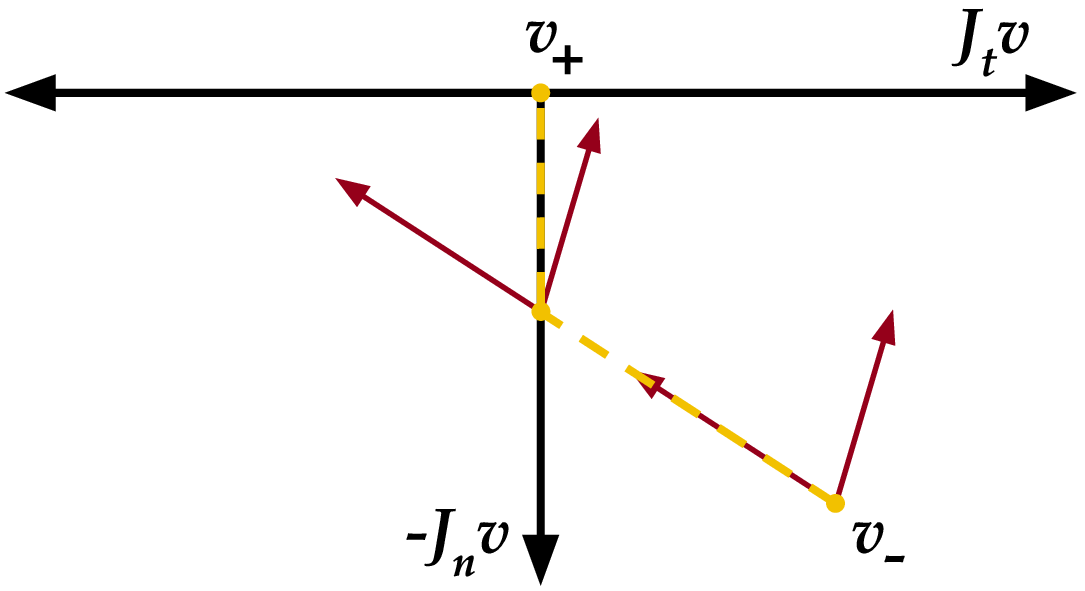}
\caption{Velocity throughout an impact resolution by Routh's method (image adapted from \citet{Posa14}).
At the initial state, the velocity-projected extreme rays of the friction cone are shown as solid arrows.
The contact begins in a sliding regime.
When $\Velocity$, shown in the dotted line, intersects $\Jt \Velocity = \ZeroVector$, the contact transitions to sticking and the impact terminates when $\Jn \Velocity=\ZeroVector$.}
\label{fig:routh}
\end{minipage}
\vspace{-5mm}
\end{figure*}

\subsection{Differential Inclusions}
The dynamics of many robots can be captured accurately with a system of ordinary differential equations (ODEs) $\dot \State = \vect{f}(\State,\Input),$ which relates $\State \in \mathbb{R}^n$, the state of the robot (typically some notion of position and velocity), to $\Input \in \mathbb{R}^m$, a set of inputs (such as motor torques) that can be manipulated. However, the dynamics of rigid bodies under frictional contact present complexities that this formulation cannot capture. Impacts between bodies induce instantaneous jumps in velocity that in general cannot described by an ODE (\textit{non-smooth} behaviors). Additionally, when contact occurs at many points, multiple frictional forces that obey Coulomb's laws of friction may exist (\textit{non-unique} behaviors). It is therefore useful to define an object that, unlike ODEs, allows for the derivative at each state to lie in a set of possible values
\begin{equation}
\dot \Velocity \in \DerivativeMap (\Velocity)\,.
\label{eq:differentialinclusion}
\end{equation}
As the map $\DerivativeMap (\Velocity)$ associated with friction may not be continuous, conditions for a function $\Velocity(t)$ to be a solution to the \textit{differential inclusion} (\ref{eq:differentialinclusion}) are weakened from those of an ODE:
\begin{definition}
	For a compact interval $I$, $\Velocity : I \to \Real^n$ is a solution to the differential inclusion $\dot \Velocity \in \DerivativeMap (\Velocity)$ if $\Velocity$ is absolutely continuous and $\dot \Velocity(t) \in \DerivativeMap (\Velocity(t))$ $a.e.$ on $I$. Denote the set of such solutions as $\SolutionSet{\DerivativeMap}[I]$.
\end{definition}
\noindent Solutions to initial value problems for (\ref{eq:differentialinclusion}) are defined similarly:
\begin{definition}
	For $I = [a, b]$ compact, denote the set of functions $\Velocity(t) \in \SolutionSet{\DerivativeMap}[I]$ with $\Velocity(a) = \Velocity_0$ as $\IVP{\DerivativeMap}{\Velocity_0}{I}$.
\end{definition}
For example, consider the differential inclusion
\begin{equation}
	\dot \Velocity \in -\Unit\Parentheses{\Velocity}\,,
\end{equation}
where $\Unit\Parentheses{\Velocity}$ is the set-valued unit direction function
\begin{equation}
	\Unit\Parentheses{\Velocity} = \begin{cases}
		\Braces{\Direction{\Velocity}} & \Velocity \neq \ZeroVector\,, \\
		\Closure \Ball[1] & \Velocity = \ZeroVector\,.
	\end{cases}
\end{equation}
For any compact interval $I = [0,T]$, the initial value problem $\IVP{-\Unit}{\Velocity_0}{I}$ admits the unique solution
\begin{equation}
	\vect s_{\Velocity_0}(t) = \begin{cases}
		\Parentheses{\TwoNorm{\Velocity_0} - t}\hat{\Velocity}_0 & t \leq \TwoNorm{\Velocity_0}\,,  \\
		\ZeroVector & t \geq \TwoNorm{\Velocity_0}\,.
	\end{cases}
\end{equation}
$\vect s_{\Velocity_0}(t)$ is non-differentiable at $t = \TwoNorm{\Velocity_0}$ and thus is not a solution of any ODE.
In general, non-emptiness, regularity, and closure of $\IVP{D}{\Velocity_0}{I}$ depend on the structure of $\DerivativeMap (\Velocity)$; fortunately, solution sets for frictional dynamics are well-behaved due to their \textit{upper semi-continuous (u.s.c.)} structure:
\begin{definition}
	A function $\DerivativeMap : A \to \PowerSet{B}$ with values closed in $B$ is \textit{upper semi-continuous} if $\forall \Sequence{a}{n} \in A,\Sequence{b}{n} \in B$ with $a_n \StrongConvergence a$, $b_n \StrongConvergence b$, and $b_n \in \DerivativeMap (a_n)$, we have $b \in \DerivativeMap (a)$. 
\end{definition}
\begin{proposition}[\citet{Aubin1984}]\label{prop:closure}
	Let $\Velocity_0 \in \VelocitySpace$ and $I$ be a compact interval. If $D(\Velocity)$ is uniformly bounded; u.s.c.; and closed, convex, and non-empty at all $\Velocity$, $\IVP{\DerivativeMap}{\Velocity_0}{I}$ is u.s.c. in $\Velocity_0$. Furthermore $\SolutionSet{\DerivativeMap}[I]$ as well as $\IVP{\DerivativeMap}{\Velocity_0}{I}$ are non-empty and closed under uniform convergence.
\end{proposition}

Intuitively, a map is u.s.c. if its value at each $\Velocity$ is not significantly smaller than its value at any $\Velocity'$ near $\Velocity$. $\Unit\Parentheses{\Velocity}$, for example, obeys all requirements of Proposition \ref{prop:closure}. As it is a singleton, $\IVP{-\Unit}{\Velocity_0}{I}$ is closed, non-empty, and convex; furthermore, if $\Velocity_n \StrongConvergence \Velocity_\infty$, then $\vect s_{\Velocity_n} \UniformConvergence \vect s_{\Velocity_\infty}$ with $\vect s_{\Velocity_\infty} \in \IVP{-\Unit}{\Velocity_\infty}{I}$.  An illustration of this system as well as the function $\Unit(\Velocity)$ can be found in Figure \ref{fig:unit}.
\begin{comment}
\subsection{SOS Programming}
In \ref{section:verification}, we will make use of sums-of-squares (SOS) programming.
Here, we provide the briefest of introductions, and refer the reader to \cite{Parrilo03a} for details.
SOS is a form of convex optimization used in polynomial optimization, with applications in robotic control and verification (e.g. \cite{Tedrake10, Posa14}).
While checking that a polynomial $p(x) \geq 0$ for all $x$ is known to be NP-hard, it is tractable to instead check that $p$ be a \textit{sum-of-squares} of some other polynomials, $p(x) = \sum q_i(x)^2$.
This can be further extended to verify that a polynomial $p(x)$ is positive on some basic semialgebraic set. That is, $p(x)\geq 0$ for all $x \in \{x : h_i(x) \leq 0, \enskip i=1,\ldots,n_h \}$, where the $h_i$'s are polynomials. 
SOS constraints for this problem can be written
\begin{align}
p(x) + \sum_i^{n_h}\sigma_i(x) h_i(x) \quad \mbox{ is SOS,}\\
\sigma_1(x),\ldots,\sigma_{n_h}(x) \quad \mbox{ are SOS,}
\end{align}
where the $\sigma_i$'s are multipliers and will be discovered by the optimization program.
While this is a sufficient condition, visible by inspection, there is a known gap between SOS and positivity, and practical guarantees of equivalence do not exist.
\end{comment}
\subsection{Frictional Impact Dynamics}
Many robots' dynamics can be modeled as a system of rigid bodies experiencing contact at up to $m$ points (for a thorough introduction, see \cite{Stewart2000} and \cite{Brogliato99}). 
The state of such a system can be represented by configuration $\Configuration(t)$ and velocities $\Velocity(t) \in \VelocitySpace$.
The continuous evolution is governed by
\begin{equation}
	\Mass(\Configuration)\dot\Velocity + \CoriolisAndGravity(\Configuration,\Velocity) = \Jn[C](\Configuration)^T\NormalForce[C] + \Jt[C](\Configuration)^T\FrictionForce[C]\,,
\end{equation}
where $\Mass(\Configuration)$ is the generalized inertial matrix; $\CoriolisAndGravity(\Configuration,\Velocity)$ encompasses Coriolis and gravitational forces; $\Jn[C] \in \Real^{\Contacts \times \States}$ projects the velocity $\Velocity$ onto the contact normals; and $\Jt[C] \in \Real^{2\FrictionalContacts \times \States}$ projects $\Velocity$ onto the contact tangents of the $k \leq m$ frictional contacts.
We identify the behavior with a set of contacts $C = \Braces{c_1,\dots,c_m}$, and identify each contact $c_i$ with its related vectors: row $i$ of $\Jn[C]$ and rows $2i-1$ and $2i$ of $\Jt[C]$, denoted as $\Jn[c_i]$ and $\Jt[c_i]$, respectively.
Denote the collection of potential contact sets as $\ContactSystems$, thus $C \in \ContactSystems$.
We furthermore define $\SizedContactSystems{m}{k} \subseteq \ContactSystems$ to be the collection of sets of $m$ contacts of which $k \leq m$ are frictional.
The world-frame contact normal and frictional forces $\NormalForce[C](t) \in \Real^{\Contacts}$ and $\FrictionForce[C](t) \in \Real^{2\FrictionalContacts}$ must lie within the Coulomb friction cone $\FrictionCone[C]\Parentheses{\Configuration,\Velocity}$; 
that is, for all $i \in \Braces{1,\dots,m}$ and $j \in \Braces{1,\dots,k}$,
\begin{align}
\NormalForce[C] & \geq \ZeroVector\,, \qquad \NormalForce[c_i]\Jn[c_i]\Velocity \leq \ZeroVector\,,\label{eq:normalseparation} \\
\FrictionForce[c_j] &\in -\FrictionCoeff[c_j]\NormalForce[c_j]\Unit\Parentheses{\Jt[c_j] \Velocity}\,,\label{eq:maximumdissipation}
\end{align}
where $\NormalForce[c_i]$ and $\FrictionForce[c_j]$ are identified similarly to $\Jn[c_i]$ and $\Jt[c_i]$ and $\FrictionCoeff[c_j] > 0$ is the friction coefficient for the $j$th contact.
Additionally, we denote the lumped terms
\begin{align}
	\J[C] &= \begin{bmatrix}
		\Jn[C] \\
		\Jt[C]
	\end{bmatrix}\,, \qquad
	\Force[C](t) = \begin{bmatrix}
		\NormalForce[C](t) \\
		\FrictionForce[C](t)
	\end{bmatrix}\,,\\
	\ActiveSet[C] &= \{ \Velocity \in \VelocitySpace : \exists c \in C, \Jn[C]\Velocity < 0\}\,,\\
	\InactiveSet[C] &= \Interior \Parentheses{\Complement{\ActiveSet[C]}} = \{ \Velocity \in \VelocitySpace : \Jn[C]\Velocity > 0\}\,.
\end{align}
$\ActiveSet[C]$ is the set of actively penetrating velocities, where impact is guaranteed to occur.
$\InactiveSet[C]$ are separating velocities, where no impact can occur.
Note that $\VelocitySpace\setminus(\ActiveSet[C] \cup \InactiveSet[C]) \neq \emptyset$, and velocities in this set \textit{may} require impacts, as in Painlev\'e's Paradox \cite{Stewart2000}.

In this work, we focus on inelastic impulsive impacts, during which velocities change instantaneously. 
Letting $\Impulse[C]$ represent an impulse, pre- and post-impact velocities, $\Velocity_-$ and $\Velocity_+$ obey
$$\Mass(\Configuration)(\Velocity_+ - \Velocity_-) = \J[C]^T\Impulse[C]\,.$$
Coulomb friction poses challenges in computing $\Impulse[C]$, as an impact may cause stick-slip transitions or change in slip direction.
For a single contact $C = \Braces c$, \citet{Routh91} proposed a graphical method describing a path in velocity space (equivalently impulse space) from $\Velocity_-$ to $\Velocity_+$ which satisfies Coulomb friction differentially.
To briefly summarize this technique,
\begin{enumerate}
    \item Increase the normal impulse $\NormalImpulse[c]$ with slope $\NormalForce[c]$.\label{item:routhnormalstep}
    \item Increment the tangential impulse $\FrictionImpulse[c]$ with slope $\FrictionForce[c]$, satisfying to Coulomb friction, identical to \eqref{eq:maximumdissipation}
    for the mid-impact velocity $\bar{\Velocity}=\Velocity_- + \Mass(\Configuration)^{-1}\J[c]^T\Impulse[c]$, the velocity after net impulse $\Impulse[c]$.\label{item:routhfrictionstep}
    \item Terminate when the normal contact velocity vanishes\footnote{To permit resolutions to Painlev\'e's Paradox, terminate only when consistency no longer requires an instantaneous change in velocity.} (i.e. $\J[N,c]\bar{\Velocity}=0$) and take $\Velocity_+ = \bar \Velocity$.\label{item:routhtermination}
\end{enumerate}
To later proceed to the multi-contact case, we observe that this process could be modeled as a u.s.c. differential inclusion:
\begin{equation}
	\dot \Velocity \in \DerivativeMap[c](\Velocity) = \begin{cases}
		\Braces\ZeroVector  & \Velocity \in \InactiveSet[c] \,,\\
		\NetForce[c](\Velocity) & \Velocity \in \ActiveSet[c]\,, \\
		\Hull \Parentheses{\Braces\ZeroVector \cup \NetForce[c](\Velocity)} & \Otherwise\,.
	\end{cases}\label{eq:routhsingle}
\end{equation}
where $F_c(\Velocity)$ is equal to the net increment in velocity due to the ``force'' applied in steps \ref{item:routhnormalstep}) and \ref{item:routhfrictionstep}) of Routh's method. Since $\Configuration$ is constant during an impact, we will apply the transformation $\Mass(\Configuration)^{.5}$ to $\Velocity$ in \eqref{eq:routhsingle}, leaving
\begin{equation}
	\NetForce[c]\Parentheses{\Velocity} = \Jn[c]^T - \FrictionCoeff[c] \Jt[c]^T \Unit\Parentheses{\Jt[c] \Velocity}\,,\label{eq:singlefrictional}
\end{equation}
where we retain the use of $\Velocity$ for ease of notation. For any  $\dot \Velocity \in \NetForce[c](\Velocity)$, we can associate a set of forces $\Force[C]$ such that
\begin{equation}
	\dot \Velocity = \J[c]^T\Force[c]\,, \qquad \NormalForce[c] = 1 \,, \qquad  \Force[c] \in \FrictionCone[c](\Configuration,\Velocity)\,.\label{eq:forceequivalence}
\end{equation}
Note that for a frictionless contact ($\FrictionCoeff[c] = 0$), this simplifies to
\begin{equation}
	\NetForce[c]\Parentheses{\Velocity} = \Braces{\Jn[c]^T}\,.\label{eq:singlefrictionless}
\end{equation}
\begin{comment}
\[
\dot \Velocity = \Mass(\Configuration)^{-1}\Force[C], \qquad \NormalForce[C] = 1 \mbox{ and } \FrictionForce[C]\in-\FrictionCoeff[C]\NormalForce[C]\Unit\Parentheses{\Jt[C] \Velocity}.
\]

\begin{equation}
\dot \Velocity = \Force, \qquad \Force[n] = 1 \mbox{ and } \Force[t]\in-\FrictionCoeff\NormalForce\Unit\Parentheses{\Jt \Velocity},
\label{eq:routh}

\end{equation}
\end{comment}
A diagram depicting the resolution of a potential planar impacts is shown in Figure~\ref{fig:routh}.
Solutions may transition between sliding and sticking, and the direction of slip may even reverse as a result of each impact.
While the path is piecewise linear in the planar case, this is not true in three dimensions.

From this point forward, we will take $s$ to be the ``simulation time'' during the resolution of an impact event; we note that evolution of $s$ \emph{does not correspond} to evolution of time, but rather measures the accumulation of impact impulse over an \emph{instantaneous} collision. In a slight abuse of notation, and we will consider total derivatives such as $\dot \Velocity(s)$ to be taken with respect to $s$.
We will also denote the impulse (i.e. the integrated force) on a contact $c$ over a sub-interval $[s_1,s_2]$ of an impact resolution as $\Impulse[c]\Parentheses{s_1,s_2}$.
Implicit in Routh's method is an assumption that the terminal condition in step \ref{item:routhtermination}) will eventually be reached by any valid choice of increment on $\Impulse[c]$; if it is possible to get ``stuck'' with $\Jn[c]\Velocity < 0$, then Routh's method would be ill-defined and not predict a post impact state. This does not happen in the frictionless case, as $\Jn[c]\Velocity$ has constant positive derivative $\Jn[c]\dot \Velocity = \TwoNorm{\Jn[c]}^2$.
The frictional case requires more careful treatment. Intuitively, the added effect of the frictional impulse will be to dissipate kinetic energy quickly. One may conclude that termination happens eventually as zero velocity is a valid post-impact state:
\begin{lemma}\label{lem:singlefrictional}
$\exists S>0$ such that for any solution $\Velocity(s) \in \SolutionSet{\DerivativeMap[c]}[\Brackets{0,\TwoNorm{\Velocity(0)}S}]$ of the single frictional contact system defined in (\ref{eq:routhsingle}) and (\ref{eq:singlefrictional}), $\exists s^* \in \Brackets{0,\TwoNorm{\Velocity(0)}S}$, $\Jn[c]\Velocity(s^*) \geq 0$.
\end{lemma}
\begin{proof}
See Appendix \ref{adx:singlestopproof}.
\end{proof}

The implication of Lemma \ref{lem:singlefrictional} is that \emph{a priori}, one can determine an $S > 0$ proportional to the pre-impact velocity $\Velocity_-$ such that any solution to the differential inclusion \eqref{eq:routhsingle} on $\Brackets{0,S}$ can be used to construct the post-impact velocity $\Velocity_+$. We will see, however, that the extension of this methodology to multiple concurrent impacts is non-trivial, and that the physicals systems associated with these models often exhibit a high degree of indeterminacy.

\section{Simultaneous Impact Model}
\label{section:model}
\subsection{Motivating Examples}
\begin{figure}[h]
    \centering
    \begin{subfigure}[b]{.18\textwidth}
        \raisebox{1cm}{\includegraphics[width=\hsize]{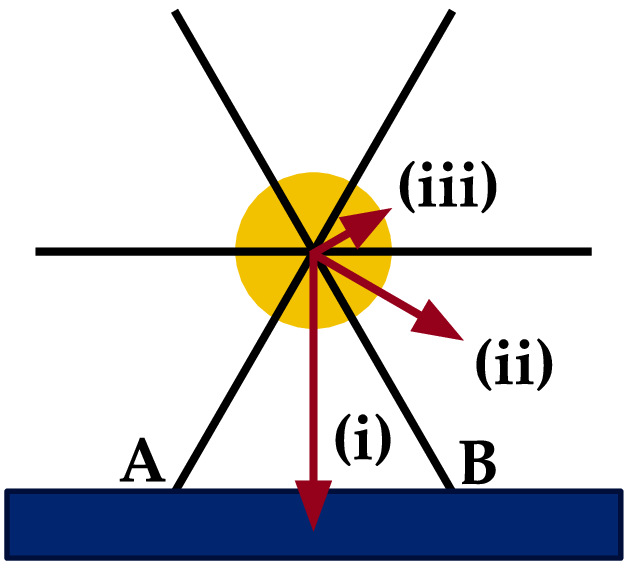}}
        \caption{\label{fig:rimless_cartoon}}
    \end{subfigure}       
%    \hspace{2cm}
    \begin{subfigure}[b]{.3\textwidth}
    	\vspace{2mm}
        \includegraphics[width=\hsize]{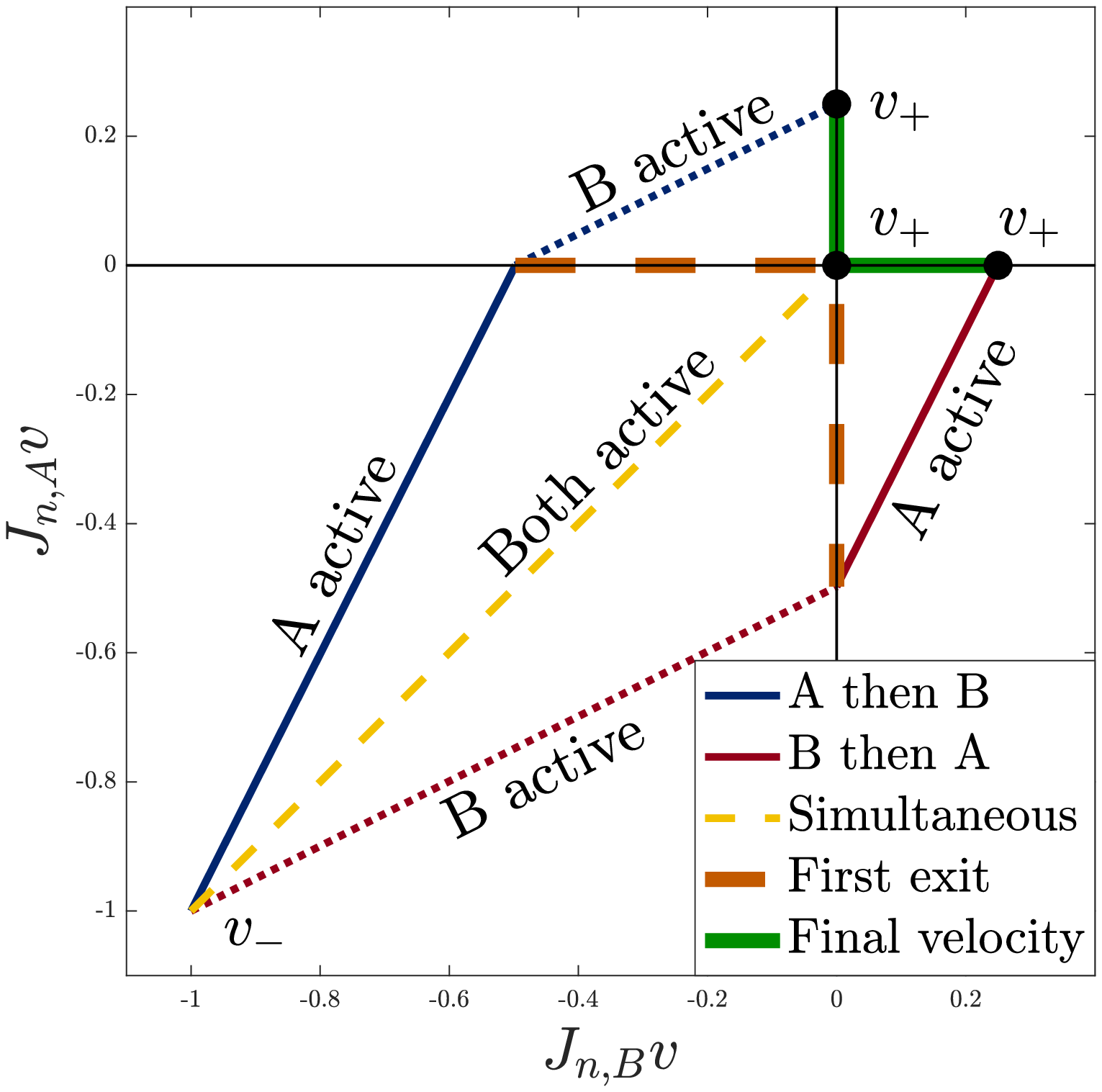}
        \caption{\label{fig:rimless_example}}
    \end{subfigure} 
    \caption{(\subref*{fig:rimless_cartoon}) Possible impact resolution for the rimless wheel with initial downward vertical velocity (i).
    A sticking impact at contact A is resolved first (ii), causing a secondary impact at B (iii).
        (\subref*{fig:rimless_example}) Impact solutions in contact \textit{normal} coordinates. Sequential resolution results one foot lifting off the ground, while simultaneous resolution results in pure sticking.
        The model defined in Section \ref{section:model} allows concurrent impacts until exiting quadrant III on the dashed orange set.
        Final post-impact velocities are shown in solid green.
    }
    \label{fig:rimless}
    %\vspace{-5mm}
\end{figure}
\label{section:examples}
We include, as motivation, two common robotics examples that exhibit simultaneous impacts: one related to legged locomotion and the other to manipulation.
Both examples, depending on initial conditions and model properties, can exhibit non-uniqueness.
Before describing our model in full detail, we present these examples by considering the outcome of applying Routh's method to a single contact at a time.  

\subsubsection{Rimless Wheel}

\begin{figure*}[h]
    \center        
    \begin{subfigure}[b]{.33\textwidth}
        \includegraphics[width=\hsize]{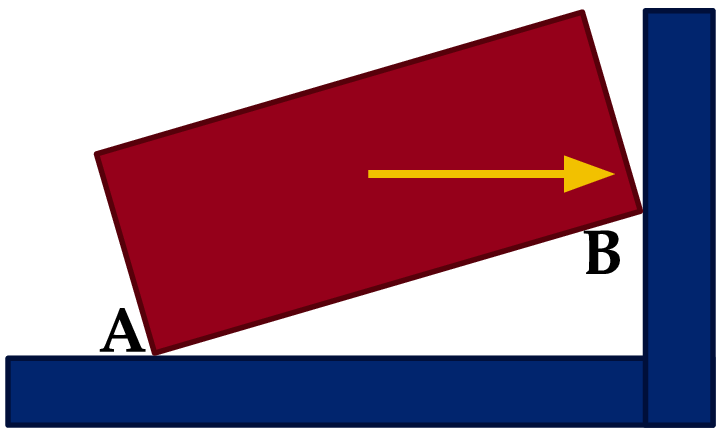}
        \caption{\label{fig:pushing_cartoon}}
    \end{subfigure}
    \begin{subfigure}[b]{.3\textwidth}
    	\vspace{2mm}
        \includegraphics[width=.9\hsize]{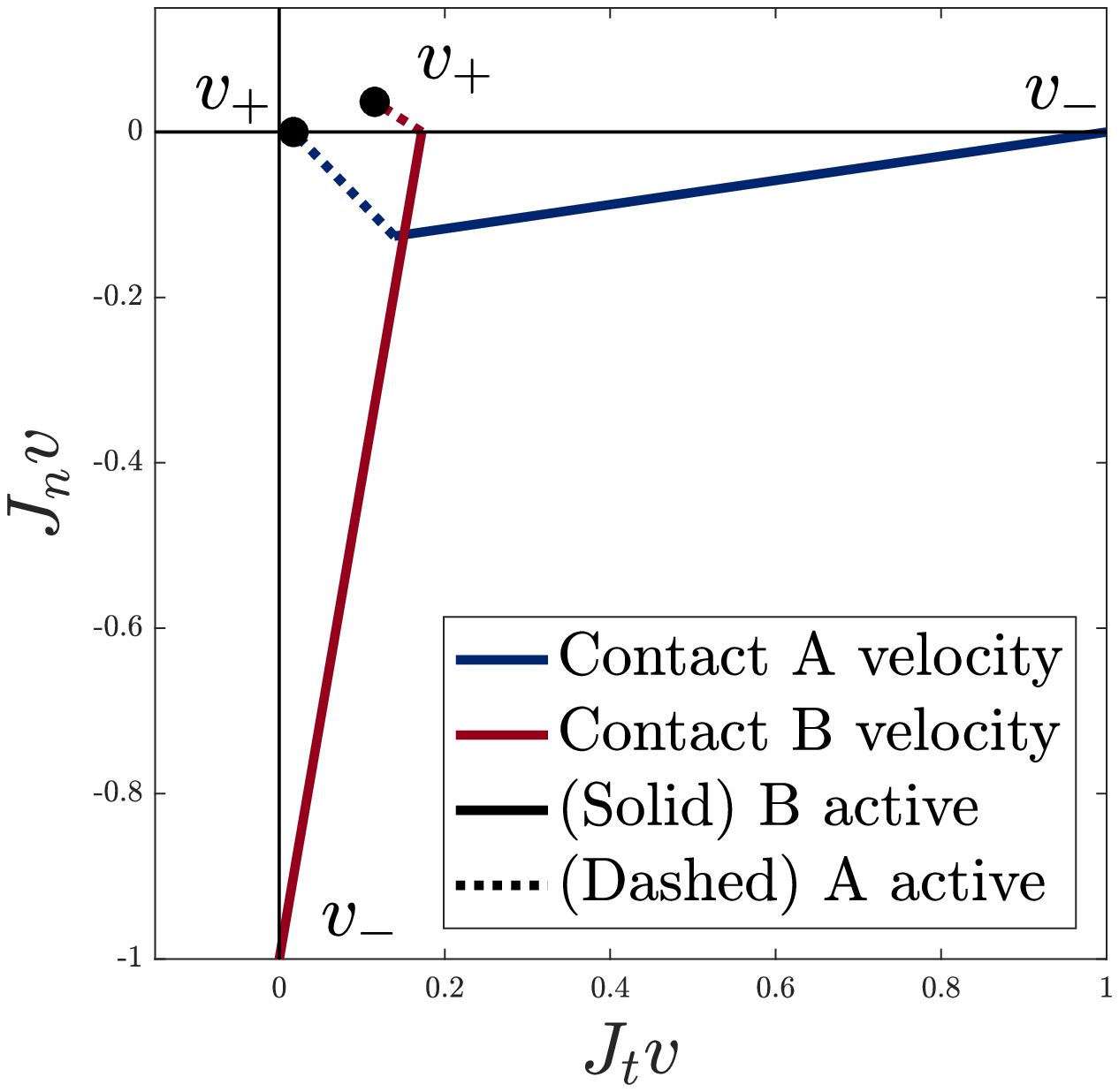}
        \caption{\label{fig:pushing_example_1}}
    \end{subfigure}
    \begin{subfigure}[b]{.3\textwidth}
    \vspace{2mm}
        \includegraphics[width=.9\hsize]{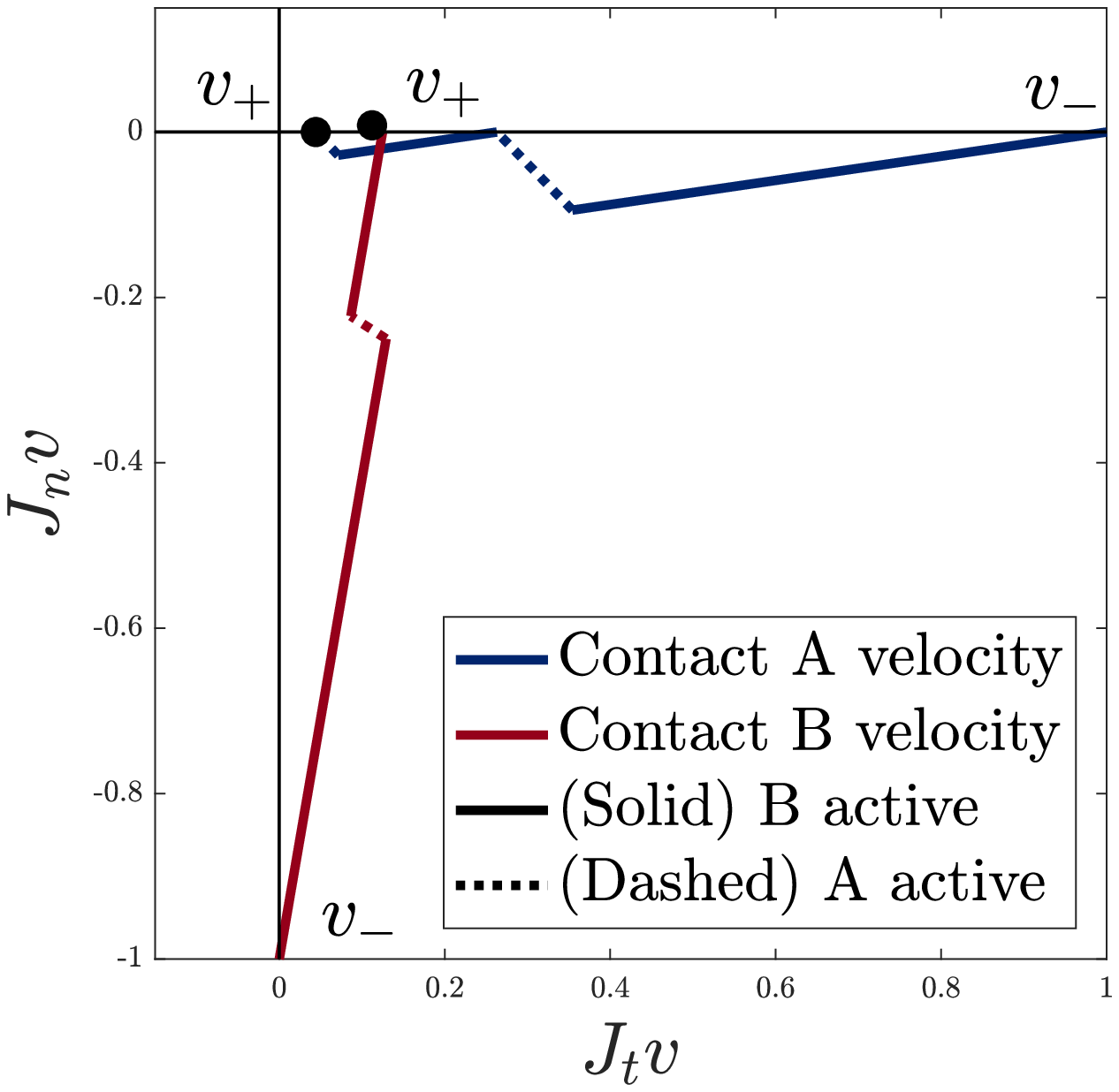}
        \caption{\label{fig:pushing_example_2}}
    \end{subfigure}
    \caption{Two subtly different solutions to a planar motion (\subref*{fig:pushing_cartoon}) are shown. 
        (\subref*{fig:pushing_example_1}) The box slides, with friction, to the right along the bottom (A) contact, before a frictionless impact on the right surface at B. This impact induces a second impact at A.
        (\subref*{fig:pushing_example_2}) Here, the incremental impulse switches to B \textit{before} the first impact terminates.}
    \label{fig:pushing}
    \vspace{-5mm}
\end{figure*}

The rimless wheel is a commonly used description of simple robotic walking \cite{Coleman97}. 
Here, we will analyze the case where two feet contact the ground.
This can occur if the robot were to fall on two feet, simultaneously, or when one foot is in sustained ground contact and the other impacts the ground.
Note that this example is not limited to a legged robot with locked hip and knee joints; see \cite{Remy2017} for a thorough analysis of similar legged examples.

For a simple example, illustrated in Fig.~\ref{fig:rimless}, we assume that both feet strike the ground vertically, with friction sufficient to sustain sticking.
In this case, existing simulation schemes (\cite{Anitescu97, Stewart1996a} and others) predict that equal impulses are generated on both feet, brining the robot to rest immediately.
However, as illustrated in the figure, if the contacts are sequenced one at a time, other post-impact states are possible where one leg separates from the ground.
For other configurations of this problem, non-unique solutions exist spanning sticking, sliding, and separation all for a single initial condition.

%If the rimless wheel is ``walking'', when the right foot impacts the ground it comes to rest and impact causes the left foot to lift off. 
%Instead, here we imagine a scenario where the right foot comes into contact with a low-friction surface.
%For simplicity, we will assume frictionless, but any sufficiently low coefficient of friction $\mu$ will suffice.
%If $\mu$ at the left foot is large, and impacts are resolved one at a time, this might lead to an infinite sequence of impact events (right, left, right, $\ldots$) but with finite total contact impulse, an example of Zeno phenomena; in this case, this sequence converges to a state with the rimless wheel at rest.
%
%However, if $\mu$ at the right foot is small, the rimless wheel will begin to slide to the right.

%\begin{figure*}
%    \center        
%    \subfloat{\includegraphics[width=0.3\textwidth]{rimless_example_2_1}\label{rimless_2_1}}
%    \hspace{.5cm}
%    \subfloat{\includegraphics[width=0.3\textwidth]{rimless_example_2_2}\label{rimless_2_2}}
%    \caption{Two possible solutions to the rimless wheel are shown. In both, the first stage consists of frictionless contact at the right foot, followed by sliding, frictional contact at the left foot. In the center figure, the first state ends when the right foot is no longer active (zero normal velocity). In the right figure, the first stage ends early.}
%\end{figure*}

\subsubsection{Nonprehensile Pushing}

In this second example, motivated by nonprehensile pushing of an object, we take a box-like object (Fig.~\ref{fig:pushing}) to have one corner sliding along a surface before impacting a frictionless second surface.
Here, the impact on the right wall causes a secondary, frictional impact against the lower wall.

If the first impact is taken to termination before activating the contact on the right wall, the solution in Fig.~\ref{fig:pushing_example_1} is discovered. Here, the bottom contact is separating and the right contact is sliding upward.
Instead, if the impact switches prior to termination, shown in Fig.~\ref{fig:pushing_example_2}, a slightly different solution emerges. 
This example illustrates that, in simple cases, reminiscent of common robotics applications, subtly different non-unique solutions can emerge from multiple contacts.

\subsection{Model Construction}
As post-impact velocity is sensitive to the ordering of individual impact resolutions, if we would like to predict as many reasonable post-impact velocities as possible, we must use as relaxed of a notion of impact resolutions as possible. 
A similar model, without theoretical results or a detailed understanding, was proposed by \citet{Posa14} where it proved useful for stability analysis of robots undergoing simultaneous impact.
We consider a formulation in which at any given instant during the resolution process, the impacts are allowed to concurrently resolve at \textit{any} relative rate:
\begin{enumerate}
	\item Monotonically increase the normal impulse on each non-separating contact $c$ at rate $\NormalForce[c] \geq 0 $ such that
		\begin{equation}
			\sum_{c \in C} \NormalForce[c] = \Norm{\NormalForce[C]}_1 = 1\,.\label{eq:convexforcecombination}
		\end{equation}
	\item Increment the tangential impulse for each frictional contact $c$ at rate $\FrictionForce[c]$ such that $
			\Force[C] \in \FrictionCone[C](\Velocity)\,.
		$
	\item Terminate when $\Velocity \not \in \ActiveSet[C]$.
\end{enumerate}
We can understand the constraint (\ref{eq:convexforcecombination}) on $\Force[C]$ as choosing a net force that comes from a convex combination of the forces that Routh's method might select for any of the individual contacts $c \in C$. As in the single contact case, we might instead think of the selection of a $\Force[C]$ as picking an element of a set of admissible values for $\dot \Velocity$. As before, we construct a u.s.c. differential inclusion to capture this behavior:
\begin{equation}
	\NetForce[C]\Parentheses{\Velocity} = \Hull \Parentheses{\Braces{f_c : c \in C, f_c \in \NetForce[c](\Velocity), \Velocity \in \Closure \ActiveSet[c] }}\,, \label{eq:permissiblenetforces}
\end{equation}
\begin{align}
	\DerivativeMap[C](\Velocity) = \begin{cases}
		\Braces\ZeroVector  & \Velocity \in \InactiveSet[C]\,, \\
		\NetForce[C]\Parentheses{\Velocity} & \Velocity \in \ActiveSet[C]\,, \\
		\Hull \Parentheses{\Braces\ZeroVector \cup \NetForce[C](\Velocity)} & \Otherwise\,.
	\end{cases}\label{eq:multicontact}
\end{align}
We denote total impulse over an interval $[s_1,s_2]$, $\Impulse[C]\Parentheses{s_1,s_2}$, as before. Similar to (\ref{eq:forceequivalence}), one can extract $\Force[C](s)$ from a solution $\Velocity(s)$ such that
\begin{equation}
	\dot \Velocity = \J[C]^T\Force[C], \qquad \Norm{\NormalForce[C]}_1 = 1\,, \qquad \Force[C] \in \FrictionCone[C](\Configuration,\Velocity)\,.\label{eq:ForceDerivative}
\end{equation}
We illustrate the behavior of this model on the rimless wheel in Figure \ref{fig:rimless_example}. While $\Velocity$ remains in the third quadrant, the direction of $\dot \Velocity$ is permitted to take any value in the convex cone outlined by the sold blue and dotted red solutions, including the simultaneous impact solution. This results in the velocity terminating at least one impact on the dashed orange set, after which behavior is identical to the single-contact system. The final velocities achievable, shown in solid green, are a superset of those given by sequential and simultaneous resolution.
\subsection{Properties}
The construction of (\ref{eq:multicontact}) is similar to that of the single contact system (\ref{eq:routhsingle}); it is furthermore equivalent when $C$ is a singleton. We now detail properties of the multi-contact system that are useful for analyzing its solution set. 
\subsubsection{Existence and Closure}
For any $C \in \ContactSystems$, $\DerivativeMap[C](\Velocity)$ is closed, uniformly bounded, and convex as it is constructed from the convex hull of a set of bounded vectors. Therefore by Proposition \ref{prop:closure}, we obtain the following: 
\begin{lemma}\label{lem:NonEmptyClosure}
	For all $C \in \ContactSystems$, velocities $\Velocity_0$, and compact intervals $I$, $\SolutionSet{\DerivativeMap[C]}[I]$ and $\IVP{\DerivativeMap[C]}{\Velocity_0}{I}$ are non-empty and closed under uniform convergence.
\end{lemma}
\subsubsection{Homogeneity}
As each $\ActiveSet[C]$ and $\InactiveSet[C]$ are conic, $\NetForce[C](\Velocity)$ and therefore $\DerivativeMap[C](\Velocity)$ are positively homogeneous in $\Velocity$. That is to say, $\forall k>0,\Velocity \in \VelocitySpace$, $\DerivativeMap[C](\Velocity) = \DerivativeMap[C](k\Velocity)$. Positive homogeneity induces a similar property on $\SolutionSet{\DerivativeMap[C]}[I]$:
\begin{lemma}[Solution Homogeneity]\label{lem:homogeneity}
	For all $C \in \ContactSystems$, $k > 0$, and compact intervals $I$, if $\Velocity(s) \in \SolutionSet{\DerivativeMap[C]}[I]$, $k\Velocity(\frac{s}{k}) \in \SolutionSet{\DerivativeMap[C]}[kI]$.
\end{lemma}
\begin{comment}
\begin{proof}
	Let $k > 0$, $\Velocity(t) \in \SolutionSet{\DerivativeMap[C]}[I]$. As $\ActiveSet[C]$ is conic for all $s \in S$, it must be that $\forall \Velocity \in \VelocitySpace$, $\DerivativeMap[C]\Parentheses{\Velocity} = \DerivativeMap[C]\Parentheses{k\Velocity}$. As $\Velocity(t) \in \SolutionSet{\DerivativeMap[C]}[I]$, $\exists E \subseteq I$ measure zero such that $\dot\Velocity(t) \in \DerivativeMap[C]\Parentheses{\Velocity(t)}$ for all $t \in I \setminus E$. $kE$ is therefore also measure zero. Let $t \in kI \setminus kE$.
	\begin{align}
		\TotalDiff{}{t}\Parentheses{k\Velocity(\frac{t}{k})} &= \dot \Velocity\Parentheses{\frac{t}{k}},\\
			&\in \DerivativeMap[C]\Parentheses{\Velocity\Parentheses{\frac{t}{k}}},\\
			&= \DerivativeMap[C]\Parentheses{k\Velocity\Parentheses{\frac{t}{k}}}.
	\end{align}
\end{proof}
\end{comment}
\subsubsection{Equivalent Minimal Coordinate Systems}
In light of (\ref{eq:ForceDerivative}), we have that $\Velocity(s) - \Velocity(s_0) \in \RangeSpace{\J[C]^T}$ for all solutions $\Velocity(s) \in \SolutionSet{\DerivativeMap[C]}[I]$. It will be useful to analyze the the evolution of a minimal-coordinate representation of $\Velocity$'s projection onto $\RangeSpace{\J[C]^T}$. Let $\vect R$ be a matrix with columns that constitute an orthogonal basis of $\RangeSpace{\J[C]^T}$. Therefore, $\vect R \vect R^T$ is an orthogonal projector onto $\RangeSpace{\J[C]^T}$ and
\begin{align}
	\J[C]\Velocity &= \Parentheses{\J[C]\vect R}\Parentheses{\vect R^T \Velocity}\,,\\
	\TotalDiff{}{t} \Parentheses{\vect R^T \Velocity} &= \Parentheses{\J[C] \vect R}^T \Force[C]\, a.e.
\end{align}
Therefore, by defining a new set of contacts $Q$ with equal size to $C$ such that $\J[Q] = \J[C]\vect R$, we have that
\begin{align}
	\Velocity(s) \in \SolutionSet{\DerivativeMap[C]}[I] &\iff \vect R^T\Velocity(s) \in \SolutionSet{\DerivativeMap[Q]}[I]\,,\\
	\Velocity \in \ActiveSet[C] &\iff \vect R^T \Velocity \in \ActiveSet[Q]\,,\\
	\J[Q]\Velocity = \ZeroVector &\iff \Velocity = \ZeroVector\,.\label{eq:Jfullrank}
\end{align}
We denote the collection of contact sets of this size that comply with the full rank condition (\ref{eq:Jfullrank}) as
\begin{equation}
	\MinimalContactSystems{m}{k} = \Braces{Q \in \SizedContactSystems{m}{k} : \J[Q] \text{ full rank}}\,.
\end{equation}
Note that $Q \in \MinimalContactSystems{k}{m}$ does not require $\J[Q]$ to have linearly independent rows; $\J[Q]$ may have more rows than columns with enough contacts, and $Q \in \MinimalContactSystems{k}{m}$ would then imply that every perturbation of $\Velocity$ would perturb at least one contact velocity.
\subsubsection{Energy Dissipation}
A basic behavior of inelastic impacts is that they dissipate kinetic energy $K(\Velocity) = \frac{1}{2}\TwoNorm{\Velocity}^2$. We now examine the dissipative properties of the model, which function both as a physical realism sanity check and as a device to prove critical theoretical properties.
On inspection of (\ref{eq:normalseparation}), (\ref{eq:maximumdissipation}) and (\ref{eq:forceequivalence}), $K(\Velocity(s))$ must be non-increasing, and furthermore, unless $\Velocity(s)$ is constant, it will strictly decrease:
\begin{lemma}[Dissipation]\label{lem:dissipate}
	Let $C \in \ContactSystems$, and let $I$ be a compact interval. If $\Velocity(s) \in \SolutionSet{\DerivativeMap[C]}[I]$, then $\TwoNorm{\Velocity(s)}$ is non-increasing.
\end{lemma}
\begin{theorem}\label{thm:stop}
	Let $C \in \ContactSystems$, and let $I$ be a compact interval. If $\Velocity(s) \in \SolutionSet{\DerivativeMap[C]}[I]$ and $\Velocity\Parentheses{I} \subseteq \ActiveSet[C]$, $\TwoNorm{\Velocity(s)}$ constant implies $\Velocity(s)$ constant.
\end{theorem}
\begin{proof}
See Appendix \ref{adx:stopproof}.
\end{proof}
One might then wonder if $K(\Velocity)$ is strictly decreasing on $\ActiveSet[C]$. One necessary condition would be $\ZeroVector \not\in \DerivativeMap[C](\Velocity^*)$ for every $\Velocity^* \in \ActiveSet[C]$, as otherwise $\Velocity(s) = \Velocity^*$ would be a solution to the differential inclusion. We will denote the collection of contacts that have this property as
\begin{equation}
	\NondegenerateContactSystems = \Braces{N \in \ContactSystems : \ZeroVector 
	\not \in \NetForce[N]\Parentheses{\ActiveSet[N]} }\,.
\end{equation}
Critically, $\NondegenerateContactSystems$ covers most situations in robotics, including grasping and locomotion, with the notable exception being jamming between immovable surfaces.
Sums-of-squares programming \cite{Parrilo03a}, a form of convex optimization, can be used to certify membership in $\NondegenerateContactSystems$.

Theorem \ref{thm:stop} and Lemma \ref{lem:dissipate} have the immediate implication that $K(\Velocity)$ strictly decreases on on $\ActiveSet[N]$ for $N \in \NondegenerateContactSystems$:
\begin{theorem}[Strict Dissipation]\label{thm:strictdissipation}
	Let $N \in \NondegenerateContactSystems$ and $I$ be a compact interval. If $\Velocity(s) \in \SolutionSet{\DerivativeMap[N]}[I]$ and $\Velocity(I) \subseteq \ActiveSet[N]$, $\TwoNorm{\Velocity(s)}$ is strictly decreasing.
\end{theorem}
\section{Finite Time Termination}
\label{section:termination}
While solutions to the underlying differential inclusion are guaranteed to exist in the multi-contact model, we have yet to prove that they terminate, as in Routh's single-contact method. Termination proofs for other simultaneous impact models (e.g. \cite{Anitescu97, Drumwright2010, Seghete2014} and others) exist, but these approaches rely on comparatively limited impulsive behaviors, and thus cannot capture essential non-unique post-impact velocities. We now show that our model exhibits what we understand to be the most permissive guaranteed termination behavior:
\begin{theorem}\label{thm:nondegeneratedissipation}
	For any pre-impact velocity $\Velocity(0)$ for a contact set $N \in \NondegenerateContactSystems$, The differential inclusion (\ref{eq:multicontact}) will resolve the impact by some $S$ proportional to $\TwoNorm{\Velocity(0)}$.
\end{theorem}

We will prove this claim as a consequence of kinetic energy decreasing fast enough to force termination---a significant expansion of Theorem \ref{thm:strictdissipation}. Even though $\KineticEnergy$ must always decrease, Theorem \ref{thm:strictdissipation} does not forbid $\TotalDiff{}{s}\KineticEnergy(\Velocity) \StrongConvergence 0$. In fact, it is not possible to create an instantaneous bound $\TotalDiff{}{s}\KineticEnergy(\Velocity) \leq -\epsilon < 0$. For example, consider 2 frictionless, axis-aligned contacts $C$ such that $\J[C] = \Identity_2$. For every $\epsilon > 0$, we can pick a velocity
	\begin{equation}
		\Velocity_\epsilon = (1+\epsilon)\begin{bmatrix}
			-1 \\
			-\epsilon
		\end{bmatrix} \in \ActiveSet[C]\,,
	\end{equation}
	\begin{equation}
		\dot \Velocity_\epsilon = \J[C]^T\begin{bmatrix}
			\epsilon \\ 1
		\end{bmatrix}\frac{1}{1+\epsilon} = \begin{bmatrix}
			\epsilon \\ 1
		\end{bmatrix}\frac{1}{1+\epsilon} \in \DerivativeMap[C](\Velocity_\epsilon)\,,
	\end{equation}
	and arrive at $\dot K > -2\epsilon$. However as we take $\epsilon \StrongConvergence 0$, $\Velocity_\epsilon$ converges to to boundary of $\ActiveSet[C]$ and thus will only be ably to sustain a small $\dot K$ for a small amount of time before terminating the impact. It remains possible that the \textit{aggregate energy dissipation} over an interval of fixed nonzero length can be bounded away from zero. We establish a rigorous characterization of this quality by defining $\DissipationRate(s)$-dissipativity:

\begin{definition}[$\DissipationRate(s)$-dissipativity]
For a positive definite function $\DissipationRate(s):\Closure \Real^+ \rightarrow [0,1)$, the system $\dot\Velocity \in \DerivativeMap[C](\Velocity)$ is said to be $\boldmath{\DissipationRate}(s)$\textbf{-dissipative} if for all $s > 0$, for all $\Velocity \in \SolutionSet{\DerivativeMap[C]}[[0,s]]$ s.t. $\Velocity\Parentheses{[0,s]} \subseteq \ActiveSet[C]$, if $\TwoNorm{\Velocity(0)} = 1$, $\TwoNorm{v(s)} \leq 1 - \DissipationRate(s)$.
\end{definition}
Denote the collection of contact sets with this property as
\begin{equation}
	\DissipativeContactSystems = \Braces{D \in \ContactSystems : \exists \DissipationRate[D](s), D\text{ is }\DissipationRate[D](s)\text{-dissipative} }\,.
\end{equation}
Intuitively, if $K > 0$ on $\ActiveSet[C]$ and $K$ decreases at a known nonzero rate, we can show that any trajectory $\Velocity(s)$ of the multi-contact system will exit $\ActiveSet[C]$ at a time linearly bounded in $\TwoNorm{\Velocity(0)}$:
\begin{lemma}[Bounded Exit]\label{lem:exit}
	Let $\DissipationRate[C](s):\Closure \Real^+ \rightarrow [0,1)$ be positive definite and let $C \in \ContactSystems$ be $\DissipationRate[C](s)$-dissipative. Then $\forall S > 0$, $\forall \Velocity(s) \in \SolutionSet{\DerivativeMap[C]}[\Brackets{0,\TwoNorm{\Velocity(0)}\frac{S}{\DissipationRate[C](S)}}]$, $\Velocity \Parentheses{\Brackets{0,\TwoNorm{\Velocity(0)}\frac{S}{\DissipationRate[C](S)}}} \not \subseteq \ActiveSet[C]$.
\end{lemma}
\begin{proof}
	See Appendix \ref{adx:exitproof}. 
\end{proof}
Any contact set $C$ that complies with the strong assumption of $\DissipationRate(s)$-dissipativity is an element of $\NondegenerateContactSystems$, as otherwise $\Velocity$ and $K$ could be constant (i.e. $\DissipativeContactSystems \subseteq \NondegenerateContactSystems$).
Far more useful is that we will show Theorem \ref{thm:nondegeneratedissipation} arises from the converse: that \textit{every} $C \in \NondegenerateContactSystems$ exhibits $\DissipationRate(s)$-dissipativity. This is particularly surprising for systems $C$ with $\J[C]$ not full rank, as $\Velocity(0)$ could be large, yet $\Velocity_R(s)$, the projection of $\Velocity(s)$ onto $\RangeSpace{\J[C]^T}$, could be arbitrary small, permitting small $\dot K$. We observe that the rank of $\J[C]$ does not effect whether or not $C \in \DissipativeContactSystems$, as all solutions will fall into two categories: either $\Velocity_R(s)$ is large, or the related minimal coordinate system will exit $\ActiveSet[C]$ very quickly:
\begin{theorem}\label{thm:minimalsufficiency}
	$\forall m > 0,m \geq k \geq 0$, $\NondegenerateContactSystems \cap \MinimalContactSystems{m}{k} \subseteq \DissipativeContactSystems \implies \NondegenerateContactSystems \cap \SizedContactSystems{m}{k} \subseteq \DissipativeContactSystems$
\end{theorem}
\begin{proof}
	See Appendix \ref{adx:minimalsufficiency}.
\end{proof}
Finally, we prove the primary claim of this work. Intuitively, if there exists $C \in \NondegenerateContactSystems$ that is not $\DissipationRate(s)$-dissipative, then one could construct a sequence of convergent solutions to $\dot \Velocity \in \DerivativeMap[C]$ that dissipate arbitrarily small amounts of energy. Therefore their limit, also a solution to $\dot \Velocity \in \DerivativeMap[C]$ as the solution set is closed, dissipates no energy---leading to a contradiction with Theorem \ref{thm:strictdissipation}. This argument will be used in an inductive manner, incrementing the size of the contact sets:
\begin{theorem}[Dissipation Inductive Step]\label{thm:incrementaldissipation}
	Assume $\NondegenerateContactSystems \cap \SizedContactSystems{m'}{k'} \subseteq \DissipativeContactSystems$ for all $m' \geq k' \geq 0$ with $k' < k$ or $k'=k$ and $m' < m$. Then $\NondegenerateContactSystems \cap \SizedContactSystems{m}{k} \subseteq \DissipativeContactSystems$.
\end{theorem}
\begin{proof}
	Suppose not. Then by Theorem \ref{thm:minimalsufficiency}, there is a set of contacts $C \in \NondegenerateContactSystems \cap \MinimalContactSystems{m}{k}$, $S > 0$, and a corresponding sequence of solutions $\Parentheses{\Velocity^j(s)}_{j \in \Natural}$, $\Velocity^j(s) \in \SolutionSet{\DerivativeMap[C]}[[0,S]]$, all starting with velocity magnitude $1$ ($\TwoNorm{\Velocity^j(0)} = 1$) and never exiting $\ActiveSet[C]$.
	We must also have that each dissipates less energy than the last: $\TwoNorm{\Velocity^j(s)} > 1 - \frac{1}{j}$. As $\DerivativeMap[C]$ is uniformly bounded, $\Velocity^j$ are unformly bounded and equicontinuous.
	By Theorem \ref{thm:Rellich} and Lemma \ref{lem:NonEmptyClosure}, we may assume that $\exists \Velocity^\infty(s) \in \SolutionSet{\DerivativeMap[C]}[[0,S]]$ such that $\Velocity^j \UniformConvergence \Velocity^\infty$. Therefore $\TwoNorm{\Velocity^\infty(s)} = 1$ for all $s$ and by Theorem \ref{thm:stop} $\Velocity^\infty$ is constant. As $C \in \NondegenerateContactSystems$, by Theorem \ref{thm:strictdissipation}, $\Velocity^\infty$ is not an element of $\ActiveSet[C]$ (i.e., $\Jn[C] \Velocity^\infty \geq \ZeroVector$). As $\J[C]$ is full rank, $\J[C] \Velocity^\infty \neq \ZeroVector$. Let $\Force[C]^j(s)$ be the corresponding force vector for each $\Velocity^j(s)$.
	
		Case 1: One contact has strictly deactivated ($\exists c \in C$, $\Jn[c] \Velocity^\infty > \ZeroVector$). But then as $\Velocity^j \UniformConvergence \Velocity^\infty$, by taking a subsequence starting from sufficiently high $j$ we may assume that $c$ \textit{never} activates ($\forall j,t,\Jn[c] \Velocity^j(s) > \ZeroVector$), and therefore at least one of the other contacts is always active ($\Velocity^j([0,S]) \subseteq \ActiveSet[C \setminus \Braces{c}]$). But then only the forces from $C \setminus \Braces{c}$ determine $\dot \Velocity^j$, and thus $\Velocity^j \in \SolutionSet{\DerivativeMap[C \setminus \Braces{c}]}[[0,S]]$. As removing a contact shrinks the set of possible forces to apply ($\NetForce[C \setminus \{c\}] \subseteq \NetForce[C]$), $C \setminus \Braces{c} \in \NondegenerateContactSystems$ and contains $m-1$ contacts. Then by assumption, for some $\DissipationRate(s)$, $C \setminus \Braces{c}$ is $\DissipationRate(s)$-dissipative. But $\TwoNorm{\Velocity^j(s)} \StrongConvergence 1$. Contradiction!
		\begin{figure}[h]
		\vspace{0.5mm}
		\centering
			\includegraphics[width=.50\hsize]{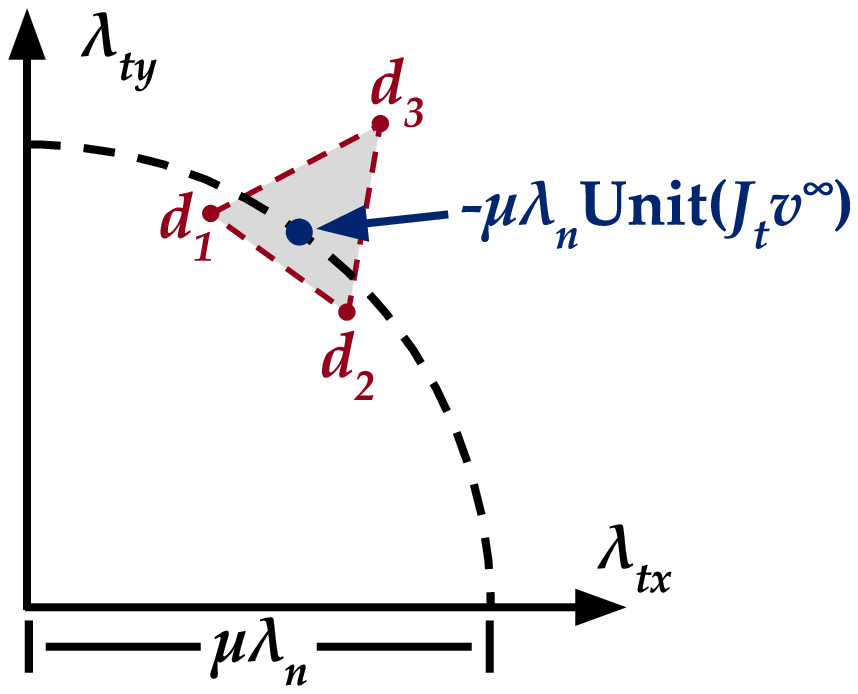}
			\caption{Conversion of a frictional contact into three frictionless contacts. As $j \StrongConvergence \infty$, we can contain $\FrictionForce[w]^j$ in an arbitrarily small neighborhood around $-\FrictionCoeff[w]\NormalForce[w]^j\Direction{\Jt[w]\Velocity^\infty}$. We pick the neighborhood to be a small triangle with vertices $\vect d_{i,w}$, such that all $\FrictionForce[w]^j$ lie in $\NormalForce[w]\Hull\Parentheses{\Braces{\vect d_{1,w},\vect d_{2,w},\vect d_{3,w}}}$, thus (\ref{eq:contactcomposition}). If the triangle is small, each $\vect d_i$ will be nearly anti-parallel to $\Jt[w]\Velocity^\infty$, implying (\ref{eq:persistence}).}
			\label{fig:contactconversion}
			
		\end{figure}
		
		Case 2: At least one contact always slides $\Parentheses{\Jt[C]\Velocity^\infty \neq \ZeroVector,\,\Jn[C]\Velocity^\infty = \ZeroVector}$.
		Let $W = \Braces{ w \in C :\TwoNorm{\Jt[w] \Velocity^\infty} > \ZeroVector} \neq \emptyset$ be the set of contacts that slide for velocity $\Velocity^\infty$.
		Then as $\Unit$ is u.s.c., $\forall w \in W,\, \Unit\Parentheses{\Jt[w] \Velocity^j} \UniformConvergence \Unit\Parentheses{\Jt[w] \Velocity^\infty}$ (i.e. convergence of the velocity to $\Velocity^\infty$ implies convergence of the direction of sliding on each contact in $W$).
		Therefore WLOG by taking a subsequence starting from sufficiently high $j$ we may assume $\forall w \in W,\, \exists \vect d_{1,w},\vect d_{2,w},\vect d_{3,w}$ sufficiently close to $-\FrictionCoeff[w]\Direction{\Jt[w] \Velocity^\infty}$ and associated new contacts $\bar w_1, \bar w_2, \bar w_3$ such that
		\begin{align}
			\Jn[\bar w_i] &= \Jn[w] + \vect d_{i,w}^T \Jt[w],\\
			\Jn[\bar w_i] \Velocity^j (s) &< 0,\label{eq:persistence} \\
			\J[w]^T\Force[w]^j (s) &\in \NormalForce[w]^j(s)\Hull\Parentheses{\bigcup_i\Braces{\Jn[\bar w_i]}} a.e.\,,\label{eq:contactcomposition}
		\end{align}
		for $i \in \Braces{1,2,3}$.
		Denote $\bar W = \bigcup_{i,w} \bar w_i$ and $\bar C = (C \cup \bar W)\setminus W$.
		(\ref{eq:persistence}) and (\ref{eq:contactcomposition}) in conjunction imply that, for velocities $\Velocity^j (s) \approx \Velocity^\infty$, each sliding frictional contact pushes mostly in one direction.
		Furthermore, the associated frictional force can be generated by three frictionless contacts tilted away from the sliding direction ($\Velocity^j \in \SolutionSet{\DerivativeMap[\bar C]}[[0,S]]$) which never deactivate ($\Velocity^j([0,S]) \subseteq \ActiveSet[\bar C]$). Figure \ref{fig:contactconversion} illustrates this construction.
		As $\bar C$ has strictly fewer frictional contacts than $C$ and is not $\DissipationRate(s)$-dissipative ($\TwoNorm{\Velocity^j(s)} \StrongConvergence 1$), by assumption we must have that $\bar C \not \in \NondegenerateContactSystems$.
		By definition of $\NondegenerateContactSystems$ there must exist some penetrating velocity $\Velocity \in \ActiveSet[\bar C]$ such that $\ZeroVector \in \NetForce[\bar C](\Velocity)$ is a permissible net force.
		We therefore must be able to find individual contact forces $\NormalForce[\bar c]f_{\bar c}$ with $\NormalForce[\bar c] \geq 0$ and $f_{\bar c} \in \NetForce[\bar c]\Parentheses{\Velocity}$ for each contact $\bar c \in \bar C$ such that $\sum_{\bar c \in \bar C} \NormalForce[\bar C] = 1$ and $\sum_{\bar c \in \bar C} \NormalForce[\bar c]f_{\bar c} = \ZeroVector$.
		As no combination of the original contacts $C$ can create zero net force alone, one of the $\bar w \in \bar W$ must strictly activate ($\NormalForce[\bar w]f_{\bar w} \neq 0$). By construction of $W$ and $\bar W$ and the assumption of Case 2, we have $\J[ C \setminus W]\Velocity^\infty = \ZeroVector$, and thus $f_{\bar c}^T\Velocity^\infty = 0$ for each $\bar c \in C \setminus W$ and $f_{\bar w}^T\Velocity^\infty < 0$ for each $\bar w \in \bar W$. Thus $\sum_{\bar c \in \bar C} \NormalForce[\bar c]f_{\bar c}^T \Velocity^\infty < 0$. But then $\sum_{\bar c \in \bar C} \NormalForce[\bar c]f_{\bar c} \neq \ZeroVector$. Contradiction!
		\end{proof}
        We are now ready to prove the main result of this section.
		\begin{proof}[Proof of Theorem \ref{thm:nondegeneratedissipation}] We will reach the claim by showing $\NondegenerateContactSystems = \DissipativeContactSystems$.
		$\NondegenerateContactSystems \supseteq \DissipativeContactSystems$ trivially. Any $C \in \MinimalContactSystems{1}{0}$ is of the form
			\begin{align}
				\DerivativeMap[C](\Velocity) = \begin{cases}
					\Braces\ZeroVector  & \InactiveSet[C] = \{j\Velocity > 0\}\\
					\Braces j & \ActiveSet[C] = \{j\Velocity < 0\} \\
					\Hull \Parentheses{\Braces{0,j}} & j\Velocity = 0
				\end{cases}
			\end{align}
			with $\Velocity, j \in \Real$, $j \neq 0$. Such a system is $\DissipationRate[C]$-dissipative with
			\begin{equation}
				\DissipationRate[C](s) = \min \Braces{\Norm{j}t,\frac{1}{2}}.
			\end{equation}
			$\SizedContactSystems{1}{0} \subseteq \DissipativeContactSystems$ by Theorem \ref{thm:minimalsufficiency}. $\NondegenerateContactSystems \subseteq \DissipativeContactSystems $ follows from nested induction on $(m,k)$ via Theorem \ref{thm:incrementaldissipation}. Therefore, $\NondegenerateContactSystems = \DissipativeContactSystems $.
		\end{proof}
%\section{Verification}
%\input{verification}
\section{Conclusion}
Non-unique behavior is a pervasive complexity that is present in both real-world robotic systems and common models capturing frictional impacts between rigid bodies---and thus accurate incorporation of such phenomena is an essential component of robust planning, control, and estimation algorithms. Our model presents a state-of-the-art theoretical foundation for the capture of this behavior, because despite the high versatility of allowing impacts to resolve at arbitrary relative rates, it is guaranteed to terminate in finite time under far more modest conditions than shown for previous models.

The logical progression from these theoretical results is to develop a numerical scheme to generate the post-impact velocity set. Constructing approximate solutions to the differential inclusion poses significant challenges associated with discontinuities in $\dot \Velocity$. While simple Euler schemes will converge to the true solution set \cite{Aubin1984}, the convergence rate is unknown, and simulation time and therefore computational complexity would scale linearly with the scale of $\TwoNorm{\Velocity_-}$ given Theorem \ref{thm:nondegeneratedissipation}. Tools from time-stepping schemes (e.g. \cite{Anitescu97,Stewart1996a}) may circumvent these issues. Another strategy is to precompute a formula for the entire post-impact set as a function of $\Velocity_-$. Sums-of-squares programming presents potential for construction of an outer approximation.

Future generalizations of the model include elastic impacts using Poisson restitution; resolution of Painlev\'e's Paradox; and a full rigid body dynamics model that has continuous solutions through impact.
\label{sec:conclusion}
\appendix
\subsection{Proof of Lemma \ref{lem:singlefrictional}}\label{adx:singlestopproof}
	Let $\vect R$ be a matrix with columns that constitute an orthogonal basis of $\RangeSpace{\J[c]^T}$. By equivalence of norms there exists $\epsilon > 0$ such that
	\begin{equation}
		\Norm{\Jn[c]\Velocity}_1 + \TwoNorm{\Jt[c]\Velocity} \geq \epsilon\TwoNorm{\vect R^T\Velocity}.
	\end{equation}
	Pick $S = \Parentheses{\epsilon\min\Parentheses{\FrictionCoeff[c],1}}^{-1}$.
	Let $V(s) = \TwoNorm{\vect R^T\Velocity(s)}^2$. Assume $\Velocity(s) \in \ActiveSet[c]$ for $s < s^* = \TwoNorm{\vect R^T\Velocity(0)}S \leq \TwoNorm{\Velocity(0)}S$. 
	\begin{align}
		\dot V &= 2 \dot \Velocity^T \vect R \vect R^T \Velocity\,,\\
		&\in  2\Parentheses{\Jn[c] - \FrictionCoeff[c] \Unit\Parentheses{\Jt[c]\Velocity}^T\Jt[c]\ } \vect R \vect R^T \Velocity\,,\\
		&= - 2\Norm{\Jn[c]\Velocity}_1 - 2\FrictionCoeff[c] \TwoNorm{\Jt[c]\Velocity}\,,\\
		&\leq -2\epsilon\min\Parentheses{\FrictionCoeff[c],1}\sqrt{V}\,,
	\end{align}
	on $[0,s^*]$ and thus $V\Parentheses{s^*} \leq \Parentheses{\sqrt{V(0)} - \epsilon\min\Parentheses{\FrictionCoeff[c],1}s^*}^2 = 0$. Therefore $\Jn[c]\Velocity\Parentheses{s^*} = 0$.

\subsection{Proof of Theorem \ref{thm:stop}}\label{adx:stopproof}
Let $\Velocity(s) \in \SolutionSet{\DerivativeMap[C]}[I]$ with $\Velocity(s)$ non-constant. Let $\Force[C](s)$ be the associated vector of force variables. As $\Velocity(s)$ is continuous, we may select $s^* \in \Interior I$ such that $\forall \delta > 0$, $\Velocity(s)$ is non-constant on $[s^*, s^* + \delta]$. Let $A = \Braces{a \in S : \Jn[a]\Velocity(s^*) \leq 0}$ be the set of active contacts at $s = s^*$. Let $B$ the the largest subset of $A$ such that $\Jn[B]\Velocity = \ZeroVector$ and $\Jt[B]\Velocity = \ZeroVector$. As $\Velocity$ is continuous, $\exists \delta_\epsilon > 0$ and $\epsilon > 0$ such that $\forall s \in [s^*,s^* + \delta_\epsilon] \subseteq I$,
	\begin{itemize}
		\item $\Jn[C \setminus A]\Velocity(s) > \epsilon$
		\item $\Jn[c]\Velocity(s) < -\epsilon$ for $c \in A \setminus B$ frictionless
		\item $\Jn[c]\Velocity(s) < -\epsilon$  or $\TwoNorm{\Jt[c]\Velocity(s)} > \frac{1}{\FrictionCoeff[i] }\epsilon$ for $c \in A \setminus B$ frictional.
	\end{itemize}
	 Therefore no new contacts activate before $s^* + \delta_\epsilon$, and
	\begin{equation}
		\Velocity(s) = \Velocity(s^*) + \J[C]^T \Impulse[C](s^*,s) = \Velocity(s^*) + \J[A]^T \Impulse[A](s^*,s)\,,
	\end{equation}
	on $[s^*,s^*+\delta_\epsilon]$. Select one such $s$ with $\Velocity(s) \neq \Velocity(s^*)$. By Lemma \ref{lem:dissipate},
	\begin{align}
		0 &\geq \frac{1}{2}\TwoNorm{\Velocity(s)}^2 - \frac{1}{2}\TwoNorm{\Velocity(s^*)}^2\,,\\
						&= \Velocity(s^*)^T\Parentheses{\Velocity(s) - \Velocity(s^*)} + \frac{1}{2}\TwoNorm{\Velocity(s) - \Velocity(s^*)}^2\,,\\
						&= \Parentheses{\J[A \setminus B]v(s^*)}^T \Impulse[A \setminus B](s^*,t) + \frac{1}{2}\TwoNorm{\Velocity(s) - \Velocity(s^*)}^2\,.
	\end{align}
	Therefore, we must have $\Norm{\Impulse[A \setminus B](s^*,t)}_1> 0$. Finally,
	\begin{align}
		K(\Velocity(s)) &= K(\Velocity(s^*)) + \int_{s^*}^{s} (\J[c]\Velocity(\tau))^T\Force[c](\tau) \Differential \tau\,, \\
			   &\leq K(\Velocity(s^*)) - \epsilon ||\Lambda_{A \setminus S}(s^*,s)||_1\,, \\  &< K(\Velocity(s^*))\,.
%			   &< K(\Velocity(s_1)).
	\end{align}
	Therefore $\TwoNorm{\Velocity}$ is non-constant.
\subsection{Proof of Lemma \ref{lem:exit}}\label{adx:exitproof}
Assume WLOG by Lemma \ref{lem:homogeneity} that $\TwoNorm{\Velocity(0)} = 1$ and that $\Velocity(s) \in \ActiveSet[C]$ on $0 \leq s < \frac{S}{\DissipationRate[C](S)}$. As $C$ is $\DissipationRate[C](s)$-dissipative, $\exists s_1 \in \Brackets{0,S}$ such that $\TwoNorm{\Velocity\Parentheses{s_1}} \leq 1 -\DissipationRate[C](S)$. A sequence $\Sequence{s}{k}$ can be iteratively constructed by Lemma \ref{lem:homogeneity} such that
		\begin{itemize}
			\item $s_k \in \Brackets{s_{k-1}, s_{k-1} + S\Parentheses{1-\DissipationRate[C](S)}^{k-1}} \subseteq \Brackets{0,\frac{S}{\DissipationRate[C](S)}}$
			\item $\TwoNorm{\Velocity\Parentheses{s_k}} \leq \Parentheses{1-\DissipationRate[C](S)}\TwoNorm{\Velocity\Parentheses{s_{k-1}}} \leq \Parentheses{1-\DissipationRate[C](S)}^k$
		\end{itemize}
		Therefore $\exists s_\infty \in \Brackets{0,\frac{S}{\DissipationRate[C](S)}}$ with $s_n \StrongConvergence s_\infty$ and by continuity of $\Velocity$, $\Velocity\Parentheses{s_\infty} = \ZeroVector \not \in \ActiveSet[C]$.
		
\subsection{Proof of Theorem \ref{thm:minimalsufficiency}}\label{adx:minimalsufficiency}
Let $C \in \NondegenerateContactSystems \cap \SizedContactSystems{m}{k}$. Let $\vect R$ and $\vect N$ be matrices with columns that constitute orthogonal bases of $\RangeSpace{\J[C]^T}$ and $\NullSpace{\J[C]}$, respectively.
	Therefore there exists contact set $Q$ of size $(m,k)$ and a positive definite function $\DissipationRate[Q](s)$ such that $\J[Q] = \J[C]\vect R$ is full column rank, $\ActiveSet[Q] = \vect R^T\ActiveSet[C]$, and $Q$ is $\DissipationRate[Q](s)$-dissipative.
	Let $s>0$, $\Velocity \in \SolutionSet{\DerivativeMap[C]}[[0,s]]$, $\TwoNorm{\Velocity(0)} = 1$, and $\Velocity\Parentheses{\Brackets{0,s}} \subseteq \ActiveSet[C]$. Decompose $\Velocity(s) = \Velocity_R(s) + \Velocity_N(s) = \vect{RR}^T \Velocity(s) + \vect{NN}^T \Velocity(0)$. We must have $\vect R^T\Velocity \in \SolutionSet{\DerivativeMap[Q]}[[0,s]]$. Therefore as $\vect R^T\Velocity([0,s]) \subseteq \vect R^T \ActiveSet[C] = \ActiveSet[Q]$, by Lemma \ref{lem:exit}, $s < \TwoNorm{\vect R^T \Velocity(0)}\frac{s}{\DissipationRate[Q](s)}$. Thus $\TwoNorm{\Velocity_R(0)} > \DissipationRate[Q](s)\TwoNorm{\Velocity(0)}$ and
	\begin{align}
		\TwoNorm{\Velocity(s)}^2 &= \TwoNorm{\Velocity_N(s)}^2 + \TwoNorm{\Velocity_R(s)}^2,\\
		&\leq \TwoNorm{\Velocity_N(0)}^2 + \Parentheses{1 - \DissipationRate[Q](s)}^2\TwoNorm{\Velocity_R(0)}^2,\\
		&\leq \TwoNorm{\Velocity(0)}^2 - \DissipationRate[Q](s)\Parentheses{1 - \DissipationRate[Q](s)}\TwoNorm{\Velocity_R(0)}^2,\\
		&\leq 1 - \DissipationRate[Q]^3(s)\Parentheses{1 - \DissipationRate[Q](s)}.
	\end{align}
	Therefore $C$ is $\Parentheses{1 - \sqrt{1 - \DissipationRate[Q]^3\Parentheses{1 - \DissipationRate[Q]}}}$-dissipative.

\Urlmuskip=0mu plus 1mu\relax
\bibliographystyle{plainnat}
\raggedright
\bibliography{library,matt}

\end{document}